\documentclass[10pt,twocolumn,letterpaper]{article}

\usepackage[pagenumbers]{iccv} 

\usepackage{algorithm}
\usepackage{algpseudocode}
\usepackage{makecell} 
\usepackage{amsthm}
\theoremstyle{plain}
\newtheorem{theorem}{Theorem}[section]

\newtheorem{definition}[theorem]{Definition}

\newcommand{\wt}{\widetilde}
\newcommand{\E}{\mathbb{E}}
\newcommand{\R}{\mathbb{R}}

\definecolor{iccvblue}{rgb}{0.21,0.49,0.74}
\usepackage[pagebackref,breaklinks,colorlinks,allcolors=iccvblue]{hyperref}

\def\paperID{2502} 
\def\confName{ICCV}
\def\confYear{2025}

\title{Discriminator-Free Direct Preference Optimization for Video Diffusion}

\author{
Haoran Cheng\textsuperscript{1}, 
Qide Dong\textsuperscript{2}, 
Liang Peng\textsuperscript{1}, 
Zhizhou Sha\textsuperscript{3},
Weiguo Feng\textsuperscript{2}, \\
Jinghui Xie\textsuperscript{2}, 
Zhao Song\textsuperscript{2}, 
Shilei Wen\textsuperscript{2}, 
Xiaofei He\textsuperscript{1}, 
Boxi Wu\textsuperscript{1}\\
\vspace{-10pt} \\
\textsuperscript{1}Zhejiang University, 
\textsuperscript{2}Bytedance, 
\textsuperscript{3}Tsinghua University\\
}

\begin{document}

\maketitle

\begin{abstract}
Direct Preference Optimization (DPO), which aligns models with human preferences through win/lose data pairs, has achieved remarkable success in language and image generation.
However, applying DPO to video diffusion models faces critical challenges:
(1) Data inefficiency—generating thousands of videos per DPO iteration incurs prohibitive costs;
(2) Evaluation uncertainty—human annotations suffer from subjective bias, and automated discriminator fail to detect subtle temporal artifacts like flickering or motion incoherence.
To address these, we propose a discriminator-free video DPO framework that:
(1) Uses original real videos as \textit{win cases} and their edited versions (e.g., reversed, shuffled, or noise-corrupted clips) as \textit{lose cases};
(2) Trains video diffusion models to distinguish and avoid artifacts introduced by editing.
This approach eliminates the need for costly synthetic video comparisons, 
provides unambiguous quality signals, 
and enables unlimited training data expansion through simple editing operations.
We theoretically prove the framework’s effectiveness even when real videos and model-generated videos follow different distributions.
Experiments on CogVideoX demonstrate the efficiency of the proposed method.
\end{abstract}
\section{Introduction}

Direct Preference Optimization (DPO)~\cite{dpo}, which leverages win/lose paired data to align model outputs with human preferences, has demonstrated remarkable success in LLMs\citep{ethayarajh2024kto, azar2024general} and text-to-image generation~\cite{DiffusionDPO, HumanFeedbackDiffusion, li2024aligning, zhu2025dspo}. 
Recent advances in video diffusion models have spurred interest in adapting DPO to video generation~\cite{stepvideot2v, videodpo, onlinevpo}. 
However, existing approaches face significant challenges in practicality and scalability.

As illustrated in Fig.~\ref{fig:pipeline}-(a), 
the DPO pipeline for video diffusion operates by first synthesizing outputs through model inference, then employing a preference discriminator to evaluate and rank these outputs based on human-aligned quality metrics. 
This process arises with two primary obstacles: 
First, \textbf{high computational costs}—generating thousands of videos per DPO iteration is prohibitively expensive (e.g., 550 seconds per 720P video for CogVideoX~\cite{yang2024cogvideox} on NVIDIA H100);
Second,  \textbf{preference discrimination} struggles with unreliable evaluation.
Human annotators often struggle with inconsistent standards for subjective video quality assessment, 
while automated methods face difficulties in consistently distinguishing subtle artifacts across video pairs. 
This discrimination challenge is further exacerbated by the narrow quality margins typically observed in generated videos, making reliable preference judgments particularly complex.

\begin{figure}[t]
  \centering
   \includegraphics[width=1\linewidth]{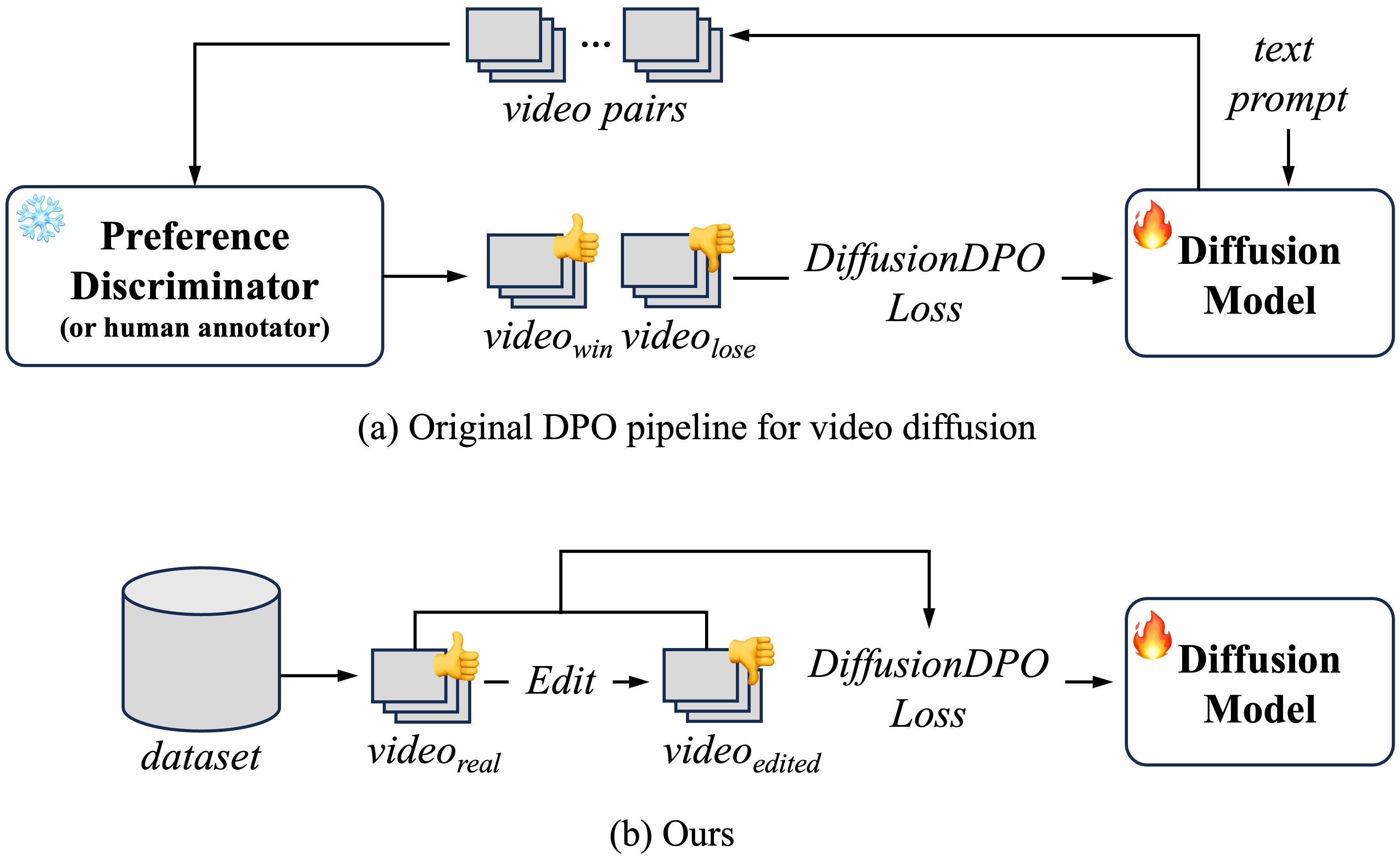}

   \caption{
   \textbf{Comparison between DPO and our proposed framework.} 
   Traditional DPO relies on computationally expensive generated video pairs, which suffer from ambiguous quality margins and scalability issues. Our method replaces generated pairs with real\&edited video pairs, where edited videos serve as lose cases, and original real videos act as win cases. This approach eliminates generative overhead, provides explicit preference signals, and enables infinite scalability.
   }
   \label{fig:pipeline}
\end{figure}

To address these challenges, we propose a novel Discriminator-Free DPO (DF-DPO) format that replaces generated video pairs with real\&edited video pairs, as illustrated in Fig.~\ref{fig:pipeline}-(b). 
Edited videos (e.g., reversed playback, frame-shuffled, or noise-corrupted real videos) serve as lose cases, while original real videos act as win cases. 
This approach offers three advantages: 
(1) Cost efficiency—real\&edited pairs eliminate generative overhead; 
(2) Explicit preference signals—editing directly introduces artifacts that models must avoid; 
(3) Infinite scalability—editing operations enable rapid dataset expansion. 

While standard DPO assumes alignment between training data and model-generated distributions, existing DPO implementations for diffusion models like DiffusionDPO~\cite{DiffusionDPO} empirically violate this principle by employing external datasets (e.g., Pick-a-Pic~\cite{pickapic}) where training and generation distributions diverge, which may cause reward misalignment and excessive regularization.
We perform analysis in Chapter~\ref{sec:theoretical_analysis}, establishing theoretical safeguards against these issues, validating our real\&edited pair paradigm as both practical and principled.

We implement our method on CogVideoX and compare it against supervised fine-tuning (SFT). Experiments demonstrate superior alignment with human preferences, validating our framework's efficacy. 

In summary, our contributions are: 
\begin{enumerate}
    \item We propose a video DPO framework using real/edited video pairs, eliminating costly generated data and ambiguous preference labels.

    \item We establish that DPO remains effective with cross-distribution training data, theoretically bridging real and generated video domains.

    \item We demonstrate the superiority of our approach through systematic comparisons with supervised fine-tuning (SFT) baselines on CogVideoX, achieving significant improvements in human preference alignment.
\end{enumerate}
\section{Related works}

\subsection{Video Diffusion Models}

The rise of diffusion models~\cite{diffusionbeatgan, Ho2020DDPM, ldm} has significantly advanced text-to-video tasks.
Some approaches~\cite{vdm, videoldm, makeavideo} inflate pre-trained T2I models by adding spatial-temporal 3D convolutions.
Several works~\cite{svd, videocrafter2} demonstrate a data-centric perspective technique to enhance the performance of T2V models.

Recent advances in generative modeling have spurred significant progress in video generation, driven by both commercial and open‐source research efforts. Commercial systems such as Sora~\cite{sora2024sora}, Gen-3~\cite{gen3runwayml}, Veo2~\cite{veo2024veo}, Kling~\cite{kling2024kling}, and Hailuo~\cite{hailuo2024hailuo} demonstrate impressive text-to-video capabilities along with extensions to image-to-video synthesis and specialized visual effects. These systems, however, typically rely on intricate pipelines with extensive pre- and post-processing. In contrast, open-source approaches like HunyuanVideo\cite{kong2024hunyuanvideo}, CogVideoX\cite{yang2024cogvideox}, Open-Sora\cite{opensora}, Open-Sora-Plan\cite{opensoraplan} and StepVideoT2V\cite{stepvideot2v} are built on transparent architectures—ranging from variations of full-attention Transformers to adaptations of DiT frameworks \cite{dit}—which not only foster community engagement but also facilitate reproducible research.

\subsection{RLHF in Generative Models}
Aligning generative models with human preferences has been a central theme in the evolution of large language models (LLMs) through techniques such as Reinforcement Learning from Human Feedback (RLHF) \citep{akrour2011preference, christiano2017deep, dubois2023alpacafarm, dubey2024llama, stiennon2020learning}. Although similar strategies have been applied to text-to-image diffusion models—leveraging supervised fine-tuning with preference data \citep{podell2023sdxl, dong2023raft, wu2023human} and reward model-based optimization \citep{clark2023directly, prabhudesai2023aligning, hao2023optimizing}—the direct adaptation of these methods to video diffusion is less explored. 

Recently, Direct Preference Optimization (DPO) \cite{dpo} has emerged as an alternative to RLHF, bypassing the need for a separate reward model training phase by directly fine-tuning the generative model with preference data. While DPO and its variants have been successfully applied in LLMs \citep{ethayarajh2024kto, azar2024general}, text-to-image diffusion models \cite{DiffusionDPO, HumanFeedbackDiffusion, li2024aligning, zhu2025dspo}, and video diffusion models \cite{stepvideot2v, videodpo, onlinevpo}.

\section{Preliminaries}

\subsection{Diffusion Models}
For diffusion models, visual contents are generated by transforming an initial noise to the desired sample through multiple sequential steps~\cite{Ho2020DDPM}. 
It is a Markov chain process where the model continually denoises the initial noise vector $\mathbf{x}_T$ and finally generates a sample $\mathbf{x}_0$. 
The generation step from $\mathbf{x}_t$ to $\mathbf{x}_{t-1}$ is given by:
\begin{equation}
    \mathbf{x}_t \sim q(\mathbf{x}_t|\mathbf{x}_{t-1}) = \mathcal{N}(\mathbf{x}_t; \sqrt{\alpha_t} \mathbf{x}_{t-1}, \beta_t \mathbf{I}),
\end{equation}
where $\beta_t$ is the variance schedule, determining the amount of noise added at each timestep $t$. $\alpha_t$ is a parameter obtained by $\alpha_t = 1 - \beta_t$ which represents the proportion of the original data retained.

The denoising model $\epsilon_\theta$, which learns to predict the noise added to $\mathbf{x}_0$ for timestep $t$, is trained by minimizing the loss between the ground-truth $\epsilon$ and prediction. The loss function is defined as
\begin{equation}
\label{eq:sftloss}
    L_d(\theta) = \mathbb{E}_{t, \mathbf{x}_0, \epsilon} \left[ \left\| \epsilon - \epsilon_\theta \left( \sqrt{\alpha_t} \mathbf{x}_0 + \sqrt{1 - \alpha_t} \epsilon, t \right) \right\|^2 \right],
\end{equation}
where $\epsilon$ is the noise added in the forward process, and $\bar{\alpha}_t$ is the cumulative product of $\alpha_t$ up to timestep $t$.


\subsection{Direct Preference Optimization}
Direct Preference Optimization~\cite{dpo} is a technique used to align generative models with human preferences. Training on pairs of generated samples with positive and negative labels, the model learns to generate positive samples with higher probability and negative samples with lower probability. DiffusionDPO adapts DPO for text-to-image diffusion models. The loss function provided in the~\cite{DiffusionDPO} is defined as:
\begin{equation}
    L_{\text{DPO}}(x^W, x^L, c) = L(x^W, p) - L(x^L, p),
\end{equation}
where $x^W$ and $x^L$ represent positive and negative samples, respectively. $L(x^W, p)$ and $L(x^L, p)$ are losses for positive and negative parts, encouraging the model to generate samples closer to preferences.
\section{Theoretical Analysis}
\label{sec:theoretical_analysis}

The foundational premise of DPO relies on an implicit assumption: the preference pairs used for training should align with the model's current generative distribution. 
However, existing implementations for diffusion models (e.g., DiffusionDPO~\cite{DiffusionDPO}) adopt a critical deviation by training on external datasets like Pick-a-Pic~\cite{pickapic}, where the training distribution inherently diverges from the model's generated outputs. 
This discrepancy raises fundamental questions about the method's theoretical validity, as distribution mismatch may induce reward miscalibration and ungrounded regularization effects.
Therefore, in this section, 
In this section, we present a theoretical analysis establishing theoretical safeguards against these issues mentioned above.
Specifically, 
Section~\ref{sec:theory:optimal_policy_guarantees} shows the objective can tell the advantage policy. 
Section~\ref{sec:theory:modeling_human_preference} demonstrates our algorithm can model human preference. 
Section~\ref{sec:theory:optmal_policy_for_video_dpo} presents the close-form of the optimal policy. 
Section~\ref{sec:theory:offsset_partition} discusses offsetting the partition function. 
For more detailed analysis, please refer to Appendix~\ref{sec:app:proof_details}.

\subsection{Optimal Policy Guarantees} \label{sec:theory:optimal_policy_guarantees}

Before delving into the theoretical details, we first outline the high-level intuition behind our analysis. The video generation process can be framed as a sequential generation task. At each timestep $t$, given a condition (or user prompt) $c$, the model generates the current frame $x^t$ conditioned on $c$ and the preceding frames $x^{<t}$. Consequently, the well-known Direct Preference Optimization (DPO) algorithm can be applied to optimize the video generation process.
We begin by introducing several key functions: the state-action function, the value function, and the advantage function, which play a central role in our subsequent proofs.
\begin{definition} [State-action function, value function, and advantage function] \label{def:three_essential_function}
If the following conditions hold:
\begin{itemize}
    \item Let $\pi$ denote a policy.
    \item Let $\gamma \in (0, 1)$ denote the discount factor. 
    \item Let $R_k$ denote the reward at timestep $k$. 
    \item Let $c$ denote the prompt used to generate the video. 
    \item Let $x^{<t}$ denote video frames generated before timestep $t$. 
    \item Let $x^t$ denote the video frame generated at timestep $t$. 
    \item Let $s_t := [c, x^{<t}]$ denote the state at timestep $t$. 
    \item Let $a_t$ denote the action taken in timestep $t$. 
\end{itemize}

We define the three essential functions as follows:
\begin{itemize}
    \item {\bf State-action function.}
    \begin{align*}
        Q_\pi([c,x^{<t}],x^t) = \E_\pi[\sum_{k=0}^{\infty}\gamma^{k} R_{t+k}|s_{t}=[c ,x^{<t}], a_t = x^t],
    \end{align*}
    \item {\bf Value function.}
    \begin{align*}
        V_\pi([c, x^{<t}]) = \E_\pi[Q_\pi([c ,x^{<t}] ,x^t)|s_t=[c,x^{<t}]],
    \end{align*}
    \item {\bf Advantage function.}
    \begin{align*}
        A_\pi([c ,x^{<t}], x^t) = Q_\pi([c, x^{<t}], c^t)-V_\pi([c, x^{<t}]).
    \end{align*}
\end{itemize}

\end{definition}

Next, we demonstarte that the value function consistently reflects the relative performance of policies. Specifically, if one policy outperforms another, it will achieve a higher expected reward as measured by the value functions.
\begin{theorem} [Optimal policy guarantees, informal version of Theorem~\ref{thm:optimal_policy_guarantees}] \label{thm:optimal_policy_guarantees:informal}
If the following conditions hold:
\begin{itemize}
    \item Let $\pi$ and $\wt\pi$ denote two policies. 
    \item Let $c$ denote the prompt used to generate the video. 
    \item Let $x^W$ denote the human-preferred generated video, and $x^L$ denote not-preferred video. 
    \item Let $Q_\pi, V_\pi, A_\pi$ denote the state-action function, value function, and the advantage function respectively, as Defined in Definition~\ref{def:three_essential_function}.
    \item Let $x^{<t}$ denote video frames generated before timestep $t$. 
    \item Let $x^t$ denote the video frame generated at timestep $t$. 
    \item Let $s_t := [c, x^{<t}]$ denote the state at timestep $t$. 
    \item Let $a_t$ denote the action taken in timestep $t$. 
    \item Suppose the policy $\wt\pi$ is better than the policy $\pi$, which means $\E_{z \sim \wt\pi}[A_\pi([c, x^{<t}], z)]\ge 0$. 
\end{itemize}

Then, we can show that
\begin{align*}
    \E_{c \sim \mathcal{D}}[V_{\wt\pi}(c)] \ge \E_{c\sim \mathcal{D}}[V_\pi(c)]. 
\end{align*}
\end{theorem}

\subsection{Modeling Human Preference} \label{sec:theory:modeling_human_preference}

Another interesting finding is that our algorithm for video generation is equivalent to the Bradley-Terry model, indicating that our method can perfectly model human preferences for videos. 
We begin with introducing the Bradley-Terry model, which quantifies human preferences by comparing the relative rewards of two videos generated from the same prompt. 
This model provides a probabilistic framework for evaluating the likelihood that one video is preferred over another based on their cumulative discounted rewards. 
We restate its formal definition as follows:
\begin{definition} [Bradley-Terry model, \cite{BradleyTerry1952}] \label{def:bradley_terry_model}
If the following conditions hold:
\begin{itemize}
    \item Let $c$ denote the prompt used to generate the video. 
    \item Let $x_1, x_2$ denote two videos generated the same prompt $c$.
    \item Let $\gamma \in (0, 1)$ denote the discount factor. 
    \item Let $r(c, x) := \sum_{t=1}^T\gamma^{t-1}R([c,x^{<t}], x^t)$ denote the reward function. 
\end{itemize}

Then, we defined the Bradley-Terry model, which measures the human preference between two videos $(x_1, x_2)$ given the same prompt $c$, as follows:
\begin{align*}
    P_{\mathrm{BT}}(x_1 \succ x_2 |c) 
    = \frac{\exp(r(c, x_1))}{\exp(r(c, x_1))+\exp(r(c, x_2))}
\end{align*}
\end{definition}
Intuitively understanding, the Bradley-Terry model measures the relative preference between two videos $(x_1, x_2)$ by comparing their cumulative discounted advantages $A_\pi$ over time steps, normalized through the logistic sigmoid function $\sigma$. 
Then, we are ready to move to showing the equivalence between Bradley-Terry model and our algorithm. 
\begin{theorem} [Equivalence with Bradley-Terry model, Theorem~\ref{thm:equivalence_to_bt_model}] \label{thm:equivalence_to_bt_model:informal}
If the following conditions hold:
\begin{itemize}
    \item Let the Bradley-Terry model be defined as Definition~\ref{def:bradley_terry_model}. 
    \item Let $Q_\pi, V_\pi, A_\pi$ denote the state-action function, value function, and the advantage function respectively, as Defined in Definition~\ref{def:three_essential_function}.
    \item Let $\sigma(x)=1/(1+\exp(-x))$ denote the logistic sigmoid function.
\end{itemize}

Then, we can show the equivalence between the Bradley-Terry model and the regret preference model as follows:
\begin{align*}
    & ~ P_{\mathrm{BT}}(x_1 \succ x_2 |c) \\
    = & ~ \sigma(\sum_{t=1}^{T_1}\gamma^{t-1}A_\pi([c,x_1^{<t}], x_1^{t}) - \sum_{t=1}^{T_2}\gamma^{t-1}A_\pi([c,x_2^{<t}], x_2^{t})). 
\end{align*}
\end{theorem}

\subsection{Optimal Policy for Video-DPO Optimization} \label{sec:theory:optmal_policy_for_video_dpo}

After demonstrating the effectiveness of our algorithm by establishing its equivalence to the Bradley-Terry model and its capability to distinguish between advantageous and disadvantageous policies, we proceed to explore the relationship between the state-action function and the optimal policy. 
The preference optimization of the video generation can be formalized into a rigorous mathematical framework, we provide the formal definition as follows:
\begin{definition} [Video-frame-level direct preference optimization problem] \label{def:video_dpo_optimization_problem}
If the following conditions hold:
\begin{itemize}
    \item Let $c$ denote the prompt used to generate the video. 
    \item Let $x^{<t}$ denote video frames generated before timestep $t$. 
    \item Let $x^t$ denote the video frame generated at timestep $t$. 
    \item Let $Q_\pi, V_\pi, A_\pi$ denote the state-action function, value function, and the advantage function respectively, as Defined in Definition~\ref{def:three_essential_function}.
    \item Let $\pi_\theta$ denote the policy being optimized.
    \item Let $\pi_{\mathrm{ref}}$ denote the reference policy.
    \item Let $\beta \in \R$ denote the hyperparameter for controlling the weight of the KL-divergence. 
\end{itemize}

Then we define the objective of the video-frame-level direct preference optimization problem as follows:
\begin{align*}
    \max_{\pi_\theta} & ~ \E_{c, x^{<t}\sim\mathcal{D},z\sim \pi_\theta(\cdot|[c,x^{<t}])}[A_{\pi_{\mathrm{
ref}}}([c,x^{<t}], z)\\
    &~ - \beta D_{\mathrm{KL}}(\pi_\theta(\cdot|[c,x^{<t}])||\pi_{\mathrm{ref}}(\cdot|[c,x^{<t}]))].
\end{align*}

\end{definition}

Based on the formal definition of the optimization problem provided above, we present our findings regarding the relationship between the state-action function and the optimal policy for the problem defined in Definition~\ref{def:video_dpo_optimization_problem}.

\begin{theorem} [Optimal policy for video-DPO problem, informal version of Theorem~\ref{thm:optimal_policy_for_video_dpo}] \label{thm:optimal_policy_for_video_dpo:informal}
If the following conditions hold:
\begin{itemize}
    \item Let the video-DPO optimization problem be defined as Definition~\ref{def:video_dpo_optimization_problem}. 
    \item Let $c$ denote the prompt used to generate the video. 
    \item Let $x^{<t}$ denote video frames generated before timestep $t$. 
    \item Let $x^t$ denote the video frame generated at timestep $t$. 
    \item Let $Q_\pi, V_\pi, A_\pi$ denote the state-action function, value function, and the advantage function respectively, as defined in Definition~\ref{def:three_essential_function}.
    \item Let $\pi_\theta$ denote the policy being optimized.
    \item Let $\pi_{\mathrm{ref}}$ denote the reference policy.
    \item Let $\beta \in \R$ denote the hyperparameter for controlling the weight of the KL-divergence. 
    \item For simplicity, let $s_t := [c, x^{<t}]$ to represent the state. 
    \item Let $Z([c,x^{<t}];\beta)$ denote the partition function, which is defined by
    \begin{align*}
        Z(s_t;\beta) := \E_{z\sim \pi_{\mathrm{ref}}(\cdot|s_t)}\exp(\beta^{-1}Q_{\pi_{\mathrm{ref}}}(s_t,z))
    \end{align*}
\end{itemize}

Then, we can show that the optimal policy satisfies the following equation:
\begin{align*}
    \pi_\theta^*(z|[c,x^{<t}])=
     \frac{\pi_{\mathrm{ref}}(z|[c,x^{<t}])\exp(\beta^{-1} Q_{\pi_{\mathrm{ref}}}([c,x^{<t}],z))}{Z([c,x^{<t}];\beta)}. 
\end{align*}
\end{theorem}

\subsection{Offsetting the Partition Function} \label{sec:theory:offsset_partition}

One key challenge with the optimal policy described above is its dependence on the partition function $ Z(s_t; \beta) $, which itself relies on the reference policy $ \pi_{\mathrm{ref}} $. This dependency prevents us from directly applying the cancellation trick used in the original DPO algorithm. However, by carefully analyzing the advantage function $ A $ and leveraging the value function $ V $, we can circumvent the limitations imposed by the partition function $ Z(s_t; \beta) $. We formalize this solution in the following theorem.
\begin{theorem} [Offset partition function $Z(s_t, \beta)$, informal version of Theorem~\ref{thm:offsetting_partition}] \label{thm:offsetting_partition:informal}
If the following conditions hold:
\begin{itemize}
    \item Let $c$ denote the prompt used to generate the video. 
    \item Let $x^{<t}$ denote video frames generated before timestep $t$. 
    \item Let $x^t$ denote the video frame generated at timestep $t$. 
    \item Let $Q_\pi, V_\pi, A_\pi$ denote the state-action function, value function, and the advantage function respectively, as defined in Definition~\ref{def:three_essential_function}.
    \item Let $\pi_\theta$ denote the policy being optimized.
    \item Let $\pi_{\mathrm{ref}}$ denote the reference policy.
    \item Let $\beta \in \R$ denote the hyperparameter for controlling the weight of the KL-divergence. 
    \item For simplicity, let $s_t := [c, x^{<t}]$ to represent the state. 
    \item Let $Z([c,x^{<t}];\beta)$ denote the partition function, which is defined by
    \begin{align*}
        Z(s_t;\beta) := \E_{z\sim \pi_{\mathrm{ref}}(\cdot|s_t)}\exp(\beta^{-1}Q_{\pi_{\mathrm{ref}}}(s_t,z))
    \end{align*}
    \item Let $u(c, x_1, x_2)$ denote the difference in rewards of two generated videos $x_1$ and $x_2$, which is defined by
    \begin{align*}
        u(c, x_1, x_2):=\beta\log\frac{\pi_{\theta}(x_1| c)}{\pi_{\mathrm{ref}}(x_1| c)}-\beta\log\frac{\pi_{\theta}(x_2| c)}{\pi_{\mathrm{ref}}(x_2| c)}
    \end{align*}
    \item Let $\delta(c, x_1, x_2)$ denote the difference in sequential forward KL divergence, which is defined by
    \begin{align*}
    \delta(c, x_1, x_2) = & ~~\beta D_{\mathrm{SeqKL}}(c,x_2;\pi_{\mathrm{ref}}\| \pi_{\theta})\\
    &~ -\beta D_{\mathrm{SeqKL}}(c,x_1;\pi_{\mathrm{ref}}\| \pi_{\theta})
    \end{align*}
\end{itemize}

Then, we can show that
\begin{align*}
    P_{\mathrm{BT}}^*(x_1 \succ x_2 |c)=\sigma(u^*(c, x_1, x_2) - \delta^*(c, x_1, x_2))
\end{align*}
\end{theorem}

Finally, we have established the rigorous theoretical framework for our algorithm. By elucidating the connections between key functions—such as the state-action function, value function, and advantage function—we demonstrate the robustness and effectiveness of our approach. This framework not only enables the algorithm to accurately model human preferences but also provides a principled method for optimizing video generation policies.



\section{Methodology}

We propose a preference discriminator-free DPO framework that replaces computationally expensive generated video pairs with real/edited video pairs. Specifically, edited videos (e.g., reversed playback, frame-shuffled, or noise-corrupted real videos) serve as lose cases, while original real videos act as win cases (detailed in Section~\ref{chapter:method1}). These win/lose pairs are then integrated into the DPO optimization process (detailed in Section~\ref{chapter:method2}).

\subsection{Discriminator-Free Data Generation}
\label{chapter:method1}

We construct real/edited video pairs through artificial distortion operations on raw videos, eliminating the need for trained discriminators. 
Let $\mathbf{V}^w = \{f_t\}_{t=1}^T$ denote the original video (win case) where $f_t \in \mathbb{R}^{H\times W\times 3}$ represents the $t$-th RGB frame, 
we generate corrupted counterparts $\mathbf{V}^l$ through three distortion categories specifically designed to simulate prevalent artifacts in video generation:

\begin{itemize}
\item \textbf{Temporal Distortion} ($\mathbf{V}^l_{\text{temp}}$):

\begin{equation}
\begin{aligned}
    \mathbf{V}^l_{\text{temp}} = \{f_{\phi(t)}\}_{t=1}^T, \\
    \phi(t) = 
        \begin{cases} 
        T+1-t & \textit{(global reversal)} \\
        \mathcal{P}(t) & \textit{(partial shuffle)}
        \end{cases}
\end{aligned}
\end{equation}

where $\mathcal{P}(t)$ denotes a random permutation operator. Specifically: {Global reversal} explicitly reverses frame order with mapping $f_t \rightarrow f_{T+1-t}$, simulating illogical motions (e.g. backward human walking). {Partial shuffle} randomly permutes frame blocks $[f_{k:m}]$ where $m-k \leq 0.2T$, creating incoherent dynamics.

 \item \textbf{Spatial Distortion} ($\mathbf{V}^l_{\text{spat}}$):
 \begin{equation}
    \mathbf{V}^l_{\text{spat}} = \{\mathcal{G}(v_t) + \epsilon_t\}_{t=1}^T,\quad \epsilon_t \sim \mathcal{N}(0,\sigma^2\mathbf{I})
 \end{equation}
 where $\mathcal{G}(\cdot)$ denotes spatial degradation operators including Gaussian blur and color shift, while $\epsilon$ adds pixel-level noise. This design follows the spatial artifact simulation principle in video codecs, where such perturbations can approximate operations like color bleeding and blocking effects. Similar strategies have proven effective in several works~\cite{simclr, zhang2017mixup, yun2019cutmixregularizationstrategytrain}.

 \item \textbf{Hybrid Distortion} ($\mathbf{V}^l_{\text{hybrid}}$):
 \begin{equation}
    \mathbf{V}^l_{\text{hybrid}} = \{\mathcal{G}(v_{\phi(t)}) + \epsilon_t\}_{t=1}^T
 \end{equation}
 This composite perturbation simultaneously injects temporal disorder through $\phi(t)$ and spatial degradation through $\mathcal{G}(\cdot)$, creating videos with coupled artifacts that mimic real-world failure modes in video generation. 
 For instance, reversing frames while applying color shifts (\textit{temporal-spatial entanglement}) forces the model to jointly address motion coherence and visual fidelity—two critical axes of video quality assessment. 
\end{itemize}

By explicitly generating these negative samples, we enforce the model to learn invariant features that resist similar artifacts.

\subsection{DPO Optimization}
\label{chapter:method2}

Our training objective combines direct preference optimization with supervised fine-tuning to leverage the complementary strengths of both paradigms:
\begin{equation}
    \mathcal{L}_{\text{total}} = \mathcal{L}_{\text{DPO}} + \lambda\mathcal{L}_{\text{SFT}}
\end{equation}

where $\lambda$ controls the balance between preference alignment and generation capability preservation. 
For video pairs $(\mathbf{V}^w, \mathbf{V}^l)$, the DPO loss amplifies the relative likelihood of the win case:
\begin{equation}
    L_{\text{DPO}}(V^W, V^L, c) = L(V^W, p) - L(V^L, p),
\end{equation}

The SFT loss can be defined as Eq.~\ref{eq:sftloss}, which anchors the model to the original data distribution, preventing over-optimization on edited artifacts.

Our framework leverages real/edited video pairs to guide preference alignment without the need for computationally expensive generated pairs.
The complete training procedure is formalized in Algorithm~\ref{alg:dfvpo}.

\begin{algorithm}[H]
\caption{Discriminator-Free Video Preference Optimization (DF-VPO)}
\label{alg:dfvpo}
\textbf{Input:} Video Set $\mathcal{V} = \{V^w_1, V^w_2, \dots, V^w_N\}$, Distortion Operators $\mathcal{D}(\cdot)$, Supervised Fine-Tuning Loss Weight $\lambda$ \\
\textbf{Output:} Preference-Aligned Video Diffusion Model $G^*(\cdot)$
\begin{algorithmic}[1]
\State Initialize video diffusion model $G(\cdot)$ with pre-trained weights
\State $step \gets 0$
\For{$V^w_i \in \mathcal{V}$}
    \State \textbf{// Edited Video Generation}
    \State $V^l_i \gets \mathcal{D}(V^w_i)$ \Comment{Generate edited video (lose case) using distortion operators}
    \State \textbf{// DPO Loss Computation}
    \State $\mathcal{L}_{\text{DPO}} \gets L(V^W_i, p) - L(V^L_i, p)$
    \State \textbf{// Supervised Fine-Tuning Loss Computation}
    \State $\mathcal{L}_{\text{SFT}} \gets \text{SupervisedLoss}(V^w_i, G)$
    \State \textbf{// Total Loss Update}
    \State $\mathcal{L}_{\text{total}} \gets \mathcal{L}_{\text{DPO}} + \lambda \mathcal{L}_{\text{SFT}}$
    \State \textbf{// Update Model Parameters}
    \State $G \gets G - \eta \nabla \mathcal{L}_{\text{total}}$ \Comment{$\eta$ is the learning rate}
    \State $step \gets step + 1$
    \If{$step \mod K == 0$}
        \State \textbf{// Curriculum Distortion Update}
        \State $\mathcal{D}(\cdot) \gets \text{UpdateDistortionOperators}(\mathcal{D}(\cdot))$ \Comment{Update distortion operators based on curriculum}
    \EndIf
\EndFor
\State \textbf{Return} $G^*(\cdot) \gets G(\cdot)$
\end{algorithmic}
\end{algorithm}


\begin{figure*}[t]
  \centering
   \includegraphics[width=1\linewidth]{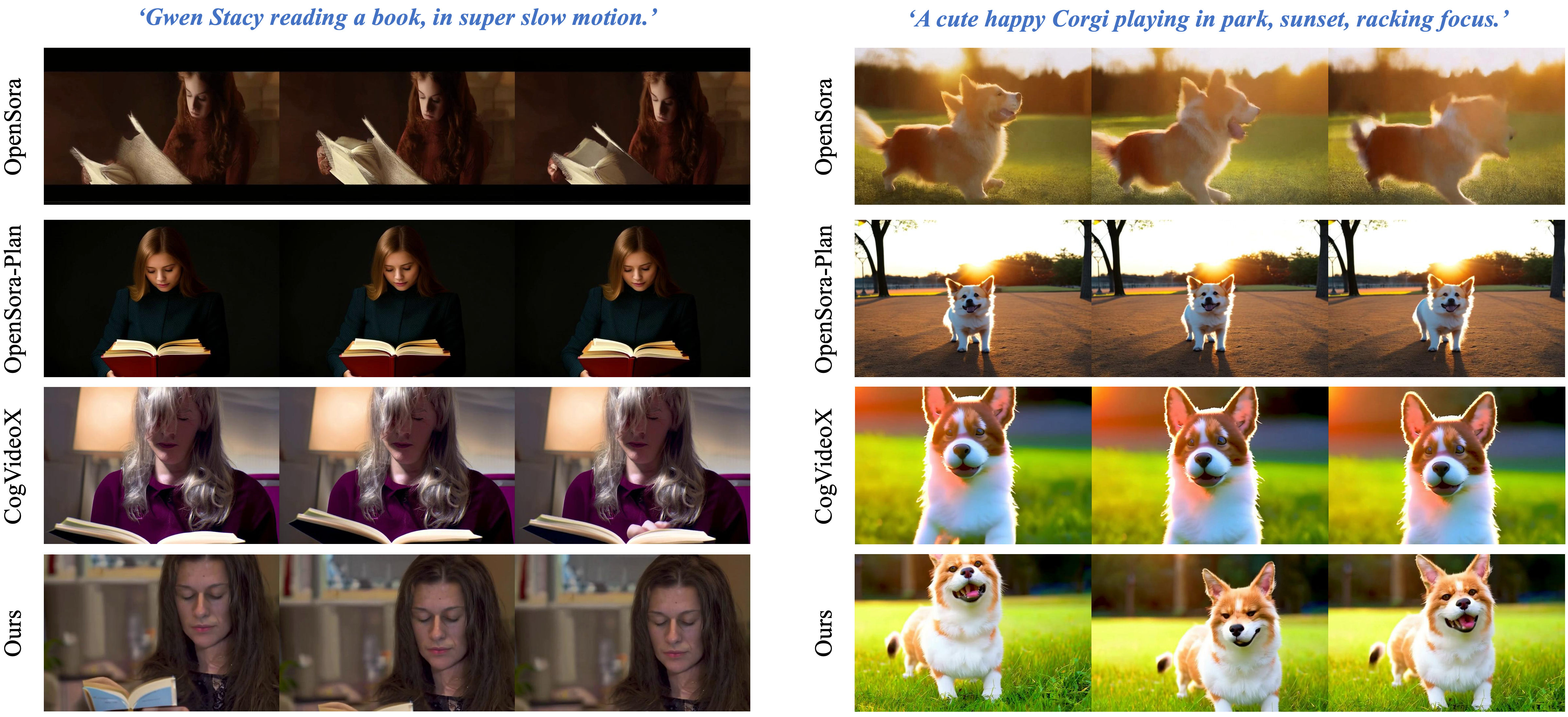}

   \caption{
    \textbf{Qualitative comparison with state-of-the-art models.} 
      Compared to OpenSora~\cite{opensora}, OpenSoraPlan~\cite{opensoraplan} and CogVideoX~\cite{yang2024cogvideox}.
      The OpenSora cases in the figure exhibit certain visual distortion, 
      while OpenSora-Plan and CogVideo cases tend to remain static. 
      In comparison, our method demonstrates good performance in both image quality and dynamic motion quality.
   }
   \label{fig:comp}
\end{figure*}


\section{Experiments}

    In this section, we perform evaluations to validate the proposed method. 
    We first describe the implementation details and dataset (Sec.~\ref{exp_implement}), 
    then compare the performance with existing methods (Sec.~\ref{exp_sota} and Sec.~\ref{exp_dpo}),
    and finally provide an analysis of the method design (Sec.~\ref{exp_abu}).

\subsection{Experiment Setup}

\paragraph{Implementation details.} 
\label{exp_implement}
Our framework is built upon CogVideoX~\cite{yang2024cogvideox} v1.0-2B model,
fine-tuned with a batch size of 1 and gradient accumulation steps of 16, trained on 8 NVIDIA H100 GPUs. 
Our reference model is the original CogVideoX model.
Due to the memory requirements of DPO training, 
which necessitates loading both the reference and training models simultaneously, 
we limit our experiments to smaller parameter models.
We use the AdamW optimizer~\cite{adamw} with a learning rate of 1e-8 and $\beta=5000$, following the DiffusionDPO~\cite{DiffusionDPO} setting. 
During inference, we generate 480P videos with 49 frames. 
We compare against two state-of-the-art video preference learning methods: Open-sora~\cite{opensora} and OpenSoraPlan~\cite{opensoraplan}.

\paragraph{Datasets.}
We use 
a publicly available open-source video-text dataset with 5 million videos and precise descriptions. It leverages a multi-scale captioning approach to ensure rich video-text alignment, supporting applications like zero-shot recognition and text-to-video generation.


\subsection{Compared with State-of-the-art Methods}
\label{exp_sota}

\paragraph{Qualitative Comparison.}
    We present qualitative comparisons of our method against SOTA baselines OpenSora~\cite{opensora} v1.3-1.1B, OpenSoraPlan~\cite{opensoraplan} v1.3.0 and CogVideoX~\cite{yang2024cogvideox} v1.0-2B.
    Results are shown in Fig. ~\ref{fig:comp}. 
    All the results are generated by the officially released models.
    In the image, we can observe the following:
    (1) OpenSora: The OpenSora cases exhibit structural distortions across critical regions. For instance, in the left frame (girl reading), facial features, hand-held books, and the right frame (Corgi), fur textures display unnatural deformations.
    (2) OpenSora-Plan and CogVideo: Both OpenSora-Plan and CogVideoX outputs show limited motion dynamics. Notably, CogVideoX introduces additional artifacts—the woman’s hair in the left case suffers from partial structural inconsistencies despite its static appearance.
    (3) Ours (DF-DPO): In contrast, our approach demonstrates superior performance in both visual fidelity and naturalistic motion dynamics. 

\begin{figure*}[t]
  \centering
   \includegraphics[width=1\linewidth]{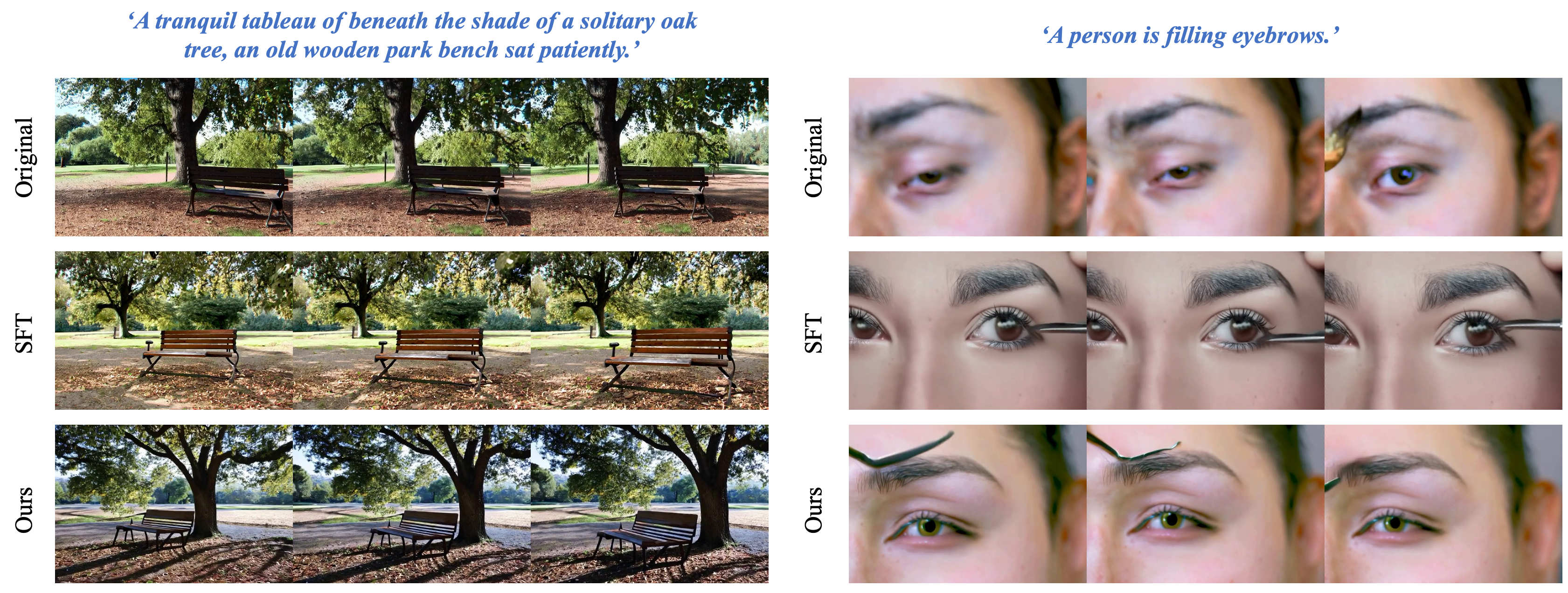}
   \caption{
       \textbf{ Comparison with SFT methods.}
       For the original model, the seat in the left case shows noticeable distortion, 
       while the right case exhibits some blurring. 
       The SFT results alleviate image quality issues but display limited motion range. 
       In contrast, our method maintains high image and motion quality while preserving a reasonable motion amplitude.
    }
   \label{fig:comp}
\end{figure*}

\begin{table*}[h]
  \caption{
    \textbf{User study results} of different models: 
    Visual Quality, Motion Quality, and Video-text Alignment ratings are on a scale from 1 to 5, with higher scores indicating better performance. 
    Our method achieves the highest scores across all evaluation criteria.
  }
  \label{userstudy}
  \centering
  \begin{tabular}{c|ccc|c}
    \toprule
    Model &
    \makecell{\textbf{Visual} \\ \textbf{Quality}} & \makecell{\textbf{Motion} \\ \textbf{Quality}} & \makecell{\textbf{Video-Text} \\ \textbf{Alignment}} &
    \textbf{Averange} 
    \\
    \midrule
    Baseline & 3.12 & 2.32 & 3.92 & 3.12 \\
    Baseline+SFT & 2.98 & 2.92 & 3.97 & 3.27 \\
    Baseline+Ours & \textbf{3.51} & \textbf{3.93} & \textbf{4.02} & \textbf{3.82} \\
    \bottomrule
  \end{tabular}
\end{table*}

\subsection{Compared with SFT Method}
\label{exp_dpo}

    \paragraph{Qualitative Comparison.}
    We present qualitative comparisons of our method against the original baseline CogVideoX~\cite{yang2024cogvideox}, SFT fine-tuned baseline.
    Results are shown in Fig. ~\ref{fig:comp}. 
    All the results are generated by the officially released models.
    In the image, we can observe the following:
    (1) Baseline: The original model exhibits visual artifacts in critical scenarios. For instance, in the park bench case, severe structural distortion occurs in the bench, while the eyebrow makeup sequence suffers from motion blurring.
    (2) SFT: SFT results alleviate image quality issues but display limited motion range. Both test cases tend toward static frames, which indicates that SFT leverages higher-quality but motion-constrained training data, likely due to the inherent motion characteristics of its training dataset.
    (3) Ours(DF-DPO): By explicitly incorporating temporal-negative samples (targeting motion artifacts) and spatial-negative samples (addressing visual quality), our approach achieves dual optimization. 
    While SFT teaches the model "what constitutes high-quality frames," the negative samples guide it to "avoid specific failure modes." 
    This mechanism enables DF-DPO trained model to generate videos with superior visual fidelity and natural motion amplitudes, striking an optimal balance between stability and dynamism.


    \paragraph{User Study.}
    For human evaluation, we conduct a user study with 30 participants to assess three key aspects of generated samples, guided by the following questions:
    (1) Visual Quality: \textit{How realistic is each static frame in the video?}
    (2) Motion Quality: \textit{Is the video almost static? Are the dynamics consistent with common human understanding? Is the motion continuous and smooth?}
    (3) Video-text Alignment: \textit{Does the video accurately reflect the target text? }
    Each question is rated on a scale from 1 to 5, with higher scores indicating better performance.  
    As shown in Tab. ~\ref{userstudy}, our method achieves the best human preferences on all evaluation parts.

\subsection{Ablation Study}
\label{exp_abu}

\begin{figure*}[t]
  \centering
   \includegraphics[width=1\linewidth]{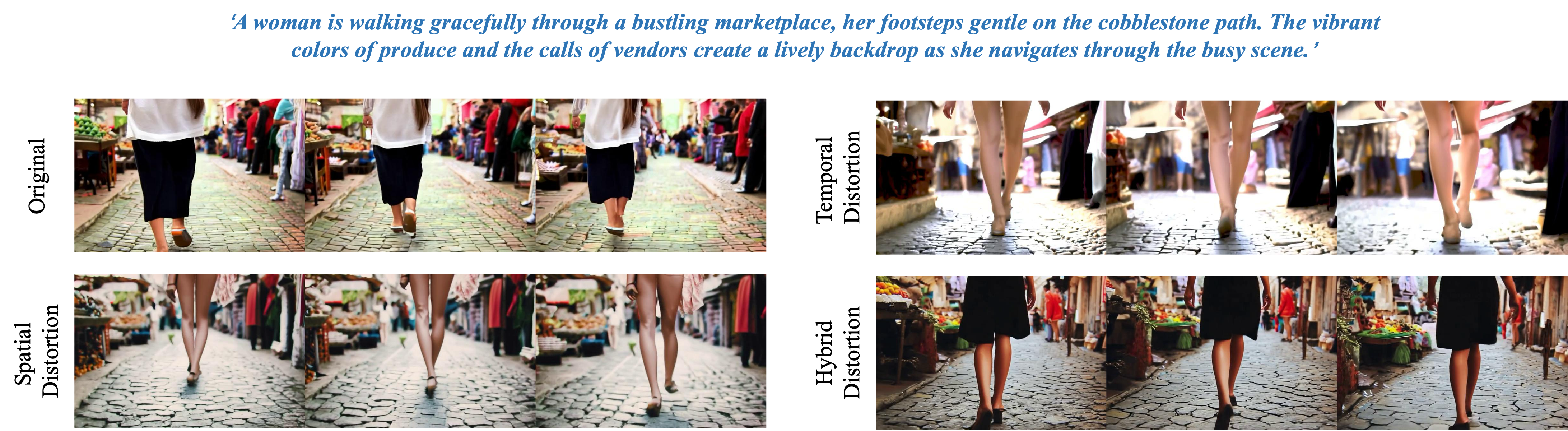}
   \caption{
        \textbf{ Comparison with different edit methods.}
        Original outputs exhibit foot distortion and motion discontinuity. Spatial Distortion improves clarity but introduces leg anomalies (frames 2-3), while Temporal Distortion enhances motion smoothness at the cost of blurring. Hybrid implementation resolves these trade-offs, achieving optimal visual-motion quality.
    }
   \label{fig:comp}
\end{figure*}

\paragraph{Comparison with other different edit methods.} 

    We perform several ablation experiments on different models of the proposed pipeline.
    The generated results are presented in Fig. ~\ref{fig:comp}.
    In the image, we can observe the following:
    (1) The original model's outputs exhibit noticeable distortion, 
    with deformations in both the feet and clothing of the character, accompanied by jerky walking motions.
    (2) Adding Spatial Distortion makes the image quality much better overall. However, the leg movements demonstrate physically implausible motion patterns—as seen in the second and third frames, where the character's leg articulation shows clear errors.
    (3) When using Temporal Distortion alone, the character's movements become smoother but the image gets blurry.
    (4) Hybrid Distortion (simultaneous integration of spatial-temporal components) delivers substantial improvements in both visual and motion quality.
\section{Conclusion}

We propose a novel discriminator-free DPO framework that eliminates the need for generated video pairs by leveraging real/edited video pairs, achieving efficient preference alignment while avoiding computational constraints. Our method demonstrates superiority over supervised fine-tuning baselines on CogVideoX (Algorithm~\ref{alg:dfvpo}), with theoretical guarantees for cross-distribution training.

\paragraph{Limitations and Future Work:}
(1) Baseline evaluations are constrained by the memory-intensive nature of CogVideoX; we will validate our framework on more efficient architectures;
(2) Current video distortions imperfectly mimic generative artifacts—future work will explore adversarial editing or learned distortion operators to better approximate real-world failure modes.
\clearpage
{
    \small
    \bibliographystyle{ieeenat_fullname}
    \bibliography{main}
}

\clearpage
\appendix 
\setcounter{figure}{0}
\setcounter{page}{1}

\section{Proof Details}
\label{sec:app:proof_details}

In this section, we present the theoretical proof details of our algorithm. Our analysis reveals that the proposed approach offers several key advantages: it leverages the sequential nature of video generation to define a structured optimization framework, aligns with human preferences through a theoretically grounded connection to the Bradley-Terry model, and ensures stable policy optimization by incorporating a principled balance between reward maximization and policy regularization. These properties collectively enhance the robustness and effectiveness of the algorithm in practical video generation tasks.

We first demonstrate that the value function consistently reflects the relative performance of policies. Specifically, if one policy outperforms another, it will achieve a higher expected reward as measured by the value functions.
\begin{theorem} [Optimal policy guarantees, formal version of Theorem~\ref{thm:optimal_policy_guarantees:informal}] \label{thm:optimal_policy_guarantees}
If the following conditions hold:
\begin{itemize}
    \item Let $\pi$ and $\wt\pi$ denote two policies. 
    \item Let $c$ denote the prompt used to generate the video. 
    \item Let $x^W$ denote the human-preferred generated video, and $x^L$ denote not-preferred video. 
    \item Let $Q_\pi, V_\pi, A_\pi$ denote the state-action function, value function, and the advantage function respectively, as Defined in Definition~\ref{def:three_essential_function}.
    \item Let $x^{<t}$ denote video frames generated before timestep $t$. 
    \item Let $x^t$ denote the video frame generated at timestep $t$. 
    \item Let $s_t := [c, x^{<t}]$ denote the state at timestep $t$. 
    \item Let $a_t$ denote the action taken in timestep $t$. 
    \item Suppose the policy $\wt\pi$ is better than the policy $\pi$, which means $\E_{z \sim \wt\pi}[A_\pi([c, x^{<t}], z)]\ge 0$. 
\end{itemize}

Then, we can show that
\begin{align*}
    \E_{c \sim \mathcal{D}}[V_{\wt\pi}(c)] \ge \E_{c\sim \mathcal{D}}[V_\pi(c)]. 
\end{align*}
\end{theorem}

\begin{proof}
We use $\tau := (c, x^1, x^2, \cdots)$ to denote the trajectory, and we use $\tau | \pi$ to denote the trajectory $\tau$ is sampled from the policy $\pi$. 

We consider the difference between $\E_{c \sim \mathcal{D}}[V_{\wt\pi}(c)]$ and $\E_{c \sim \mathcal{D}}[V_\pi(c)]$. We have the following
\begin{align} \label{eq:better_policy}
    & ~ \E_{c \sim \mathcal{D}}[V_{\wt\pi}(c)] - \E_{c\sim \mathcal{D}}[V_\pi(c)] \notag \\
    = & ~\E_{\tau|\wt\pi}[\sum_{t=1}^{\infty}\gamma^{t-1} R_{t}-V_\pi(c)] \notag \\
    = & ~\E_{\tau|\wt\pi}[\sum_{t=1}^{\infty}\gamma^{t-1} (R_{t}+\gamma V_\pi([c, x^{<t+1}])-V_\pi([c, x^{<t}]))] \notag \\
    = & ~\E_{\tau|\wt\pi}[\sum_{t=1}^{\infty}\gamma^{t-1} A_\pi([c, x^{<t}], x^t)] \notag \\
    = & ~\E_{\tau|\wt\pi}[\sum_{t=1}^{\infty}\gamma^{t-1} \E_{x^t \sim \wt\pi}[A_\pi([c, x^{<t}], x^t)]] \notag \\
    \geq & ~ 0
\end{align}
where the first step follows from the definition of the value function $V$, the second step follows from the definition of the reward $R_t$, the third step follows from the definition of the advantage function $A$, the fourth step reformulates the terms in to expectation format, the fifth step follows from $\E_{z \sim \wt\pi}[A_\pi([c, x^{<t}], z)]\ge 0$, which is mentioned in the conditions of this lemma. 

Reformulate Eq.~\eqref{eq:better_policy}, we have
\begin{align*}
    \E_{c \sim \mathcal{D}}[V_{\wt\pi}(c)] - \E_{c\sim \mathcal{D}}[V_\pi(c)] \ge 0. 
\end{align*}

The final result can be obtained by shifting the terms in the equation. 

\end{proof}

Then, we move to showing the equivalence between Bradley-Terry model and our algorithm. 
\begin{theorem} [Equivalence with Bradley-Terry model, formal version of Theorem~\ref{thm:equivalence_to_bt_model:informal}] \label{thm:equivalence_to_bt_model}
If the following conditions hold:
\begin{itemize}
    \item Let the Bradley-Terry model be defined as Definition~\ref{def:bradley_terry_model}. 
    \item Let $Q_\pi, V_\pi, A_\pi$ denote the state-action function, value function, and the advantage function respectively, as Defined in Definition~\ref{def:three_essential_function}.
    \item Let $\sigma(x)=1/(1+\exp(-x))$ denote the logistic sigmoid function.
\end{itemize}

Then, we can show the equivalence between the Bradley-Terry model and the regret preference model as follows:
\begin{align*}
    & ~ P_{\mathrm{BT}}(x_1 \succ x_2 |c) \\
    = & ~ \sigma(\sum_{t=1}^{T_1}\gamma^{t-1}A_\pi([c,x_1^{<t}], x_1^{t}) - \sum_{t=1}^{T_2}\gamma^{t-1}A_\pi([c,x_2^{<t}], x_2^{t})). 
\end{align*}
\end{theorem}

\begin{proof}
According to the definition of Bradley-Terry model (Definition~\ref{def:bradley_terry_model}), we have 
\begin{align} \label{eq:bt_model_definition}
    P_{\mathrm{BT}}(x_1 \succ x_2 |c) 
    = \frac{\exp(r(c, x_1))}{\exp(r(c, x_1))+\exp(r(c, x_2))}
\end{align}

Before delving into the details of the proof, we first present two useful equations that will facilitate the subsequent analysis.

Since the video generation process can be viewed as a sequential generation. Therefore the transition to the next frame generation is deterministic when given the current state and action. Namely we have, $\Pr(s_{t+1}=[c,x^{<t+1}]|s_{t}=[c,x^{<t}], a_t=x^t) = 1$, so we have:
\begin{align*}
    Q_\pi([c,x^{<t}], x^t) = & ~ R([c,x^{<t}], x^t) + V_\pi([c,x^{<t+1}])
\end{align*}
and
\begin{align*}
    A_\pi([c,x^{<t}], x^t) = & ~ Q_\pi([c,x^{<t}], x^t) - V_\pi([c,x^{<t}])
\end{align*}

We use $x^T$ to denote the last frame of the generated video. Then, we have the following
\begin{align} \label{eq:last_value_function_is_zero}
    & ~ V_\pi([c,x^{<T+1}]) \notag \\
    = & ~ \E_\pi[\sum_{k=0}^{\infty}\gamma^{k} R([c,x^{<T+1+k}], x^{T+1+k})|s_{t}=[c,x^{<T+1}]] \notag \\
    = & ~ 0
\end{align}

According to the definition of $x^{<t}$, $x^{<1}$ represents the empty set. Then we can derive the following
\begin{align} \label{eq:empty_set}
    & ~ V_\pi([c,x_1^{<1}]) \\
    = & ~ V_\pi([c,[\ ]]) \\
    = & ~ V_\pi([c,x_2^{<1}])
\end{align}
where the first step follows from $x_1^{<1}$ represents the empty set, the second step follows from $x_2^{<1}$ represents the empty set. 

With the two critical math tools derived above. Then, we can derive the following equations regarding the reward function $r(c, x)$,
\begin{align} \label{eq:bt_model_equivalence}
    & ~ r(c, x) \notag \\
    = & ~ \sum_{t=1}^{T}\gamma^{t-1}R([c,x^{<t}], x^t) \notag \\
    = & ~\sum_{t=1}^{T} \gamma^{t-1}(R([c,x^{<t}], x^t) + \gamma V_\pi([c,x^{<t+1}]) \notag \\
    & ~ - \gamma V_\pi([c,x^{<t+1}])) \notag \\
    = & ~V_\pi([c,x^{<1}]) + \sum_{t=1}^{T}\gamma^{t-1}( R([c,x^{<t}], x^t) + \gamma V_\pi([c,x^{<t+1}]) \notag \\
    & ~ - V_\pi([c,x^{<t}])) - \gamma^T V_\pi([c,x^{<T+1}])
\end{align}
where the first step follows from the definition the reward function $r(c, x)$, the second step follows from basic algebra, the third step follows from Eq.~\eqref{eq:last_value_function_is_zero} and extracting the $V_\pi([c,x^{<1}])$ from the summation.  

Combining Eq.~\eqref{eq:bt_model_definition} and Eq.~\eqref{eq:bt_model_equivalence}, we have
\begin{align*}
    & ~ P_{\mathrm{BT}}(x_1\succ x_2 | c) \\
    = & ~\frac{\exp(r(c, x_1))}{\exp(r(c, x_1))+\exp(r(c, x_2))} \\
    = & ~ ~\sigma((V_\pi([c,x_1^{<1}]) \\
    + & ~ \sum_{t=1}^{T_1}(\gamma^{t-1} A_\pi([c,x_1^{<t}], x^t)))-(V_\pi([c,x_2^{<1}]) \\
    + & ~ \sum_{t=1}^{T_2}(\gamma^{t-1} A_\pi([c,x_2^{<t}], x_2^t)))) \\
    = & ~\sigma(\sum_{t=1}^{T_1}(\gamma^{t-1} A_\pi([c,x_1^{<t}], x_1^t)) \\
    - & ~ \sum_{t=1}^{T_2}(\gamma^{t-1} A_\pi([c,x_2^{<t}], x_2^t)))
\end{align*}
where the first step follows from the definition of Bradley-Terry model, the second step follows from integrating Eq.~\eqref{eq:bt_model_equivalence} to Eq.~\eqref{eq:bt_model_definition}, the last step follows from Eq.~\eqref{eq:empty_set}. 

Therefore, we have shown the equivalence between Bradley-Terry model and our algorithm. 
\end{proof}

Based on the formal definition of the optimization problem provided above, we present our findings regarding the relationship between the state-action function and the optimal policy for the problem defined in Definition~\ref{def:video_dpo_optimization_problem}.
\begin{theorem} [Optimal policy for video-DPO problem, formal version of Theorem~\ref{thm:optimal_policy_for_video_dpo:informal}] \label{thm:optimal_policy_for_video_dpo} 
If the following conditions hold:
\begin{itemize}
    \item Let the video-DPO optimization problem be defined as Definition~\ref{def:video_dpo_optimization_problem}. 
    \item Let $c$ denote the prompt used to generate the video. 
    \item Let $x^{<t}$ denote video frames generated before timestep $t$. 
    \item Let $x^t$ denote the video frame generated at timestep $t$. 
    \item Let $Q_\pi, V_\pi, A_\pi$ denote the state-action function, value function, and the advantage function respectively, as defined in Definition~\ref{def:three_essential_function}.
    \item Let $\pi_\theta$ denote the policy being optimized.
    \item Let $\pi_{\mathrm{ref}}$ denote the reference policy.
    \item Let $\beta \in \R$ denote the hyperparameter for controlling the weight of the KL-divergence. 
    \item For simplicity, let $s_t := [c, x^{<t}]$ to represent the state. 
    \item Let $Z([c,x^{<t}];\beta)$ denote the partition function, which is defined by
    \begin{align*}
        Z(s_t;\beta) := \E_{z\sim \pi_{\mathrm{ref}}(\cdot|s_t)}\exp(\beta^{-1}Q_{\pi_{\mathrm{ref}}}(s_t,z))
    \end{align*}
\end{itemize}

Then, we can show that the optimal policy satisfies the following equation:
\begin{align*}
    \pi_\theta^*(z|[c,x^{<t}])=
     \frac{\pi_{\mathrm{ref}}(z|[c,x^{<t}])\exp(\beta^{-1} Q_{\pi_{\mathrm{ref}}}([c,x^{<t}],z))}{Z([c,x^{<t}];\beta)}. 
\end{align*}
\end{theorem}

\begin{proof}


Then, we can derive the following equations regarding to the optimization problem defined in Definition~\ref{def:video_dpo_optimization_problem},
\begin{align} \label{eq:optimal_policy}
& ~\max_{\pi_\theta}  \E_{z\sim \pi_\theta(\cdot|s_t)}A_{\pi_{\mathrm{ref}}}(s_t, z)  - \beta D_{\mathrm{KL}}(\pi_\theta(\cdot|s_t)\|\pi_{\mathrm{ref}}(\cdot|s_t))\notag\\
= & ~\max_{\pi_\theta}\ \E_{z\sim \pi_\theta(\cdot|s_t)}((Q_{\pi_{\mathrm{ref}}}(s_t,z) -V_{\pi_{\mathrm{ref}}}(s_t))\notag\\
& ~+\beta\log (\frac{\pi_{\mathrm{ref}}(z|s_t)}{\pi_\theta(z|s_t)}))\notag\\
= & ~\max_{\pi_\theta}\ \beta\E_{z\sim \pi_\theta(\cdot|s_t)}\log(\frac{p(z|s_t)\exp(\beta^{-1}Q_{\pi_{\mathrm{ref}}}(s_t,z))}{\pi_\theta(z|s_t)}) \\
&~-V_{\pi_{\mathrm{ref}}}(s_t)\notag\\
= & ~\max_{\pi_\theta}\ \beta\E_{z\sim \pi_\theta(\cdot|s_t)}\log(\frac{\pi_{\mathrm{ref}}(z|s_t)\exp(\beta^{-1}Q_{\pi_{\mathrm{ref}}}(s_t,z))}{Z(s_t;\beta)\pi_\theta(z|s_t)})\notag \\
& ~ -V_{\pi_{\mathrm{ref}}}(s_t)+\beta\log Z(s_t;\beta)\notag\\
= & ~\max_{\pi_\theta}  -\beta D_{\mathrm{KL}}(\pi_\theta(z|s_t)\|\frac{\pi_{\mathrm{ref}}(z|s_t)\exp(\beta^{-1}Q_{\pi_{\mathrm{ref}}}(s_t,z))}{Z(s_t;\beta)}) \notag\\
& ~ -V_{\pi_{\mathrm{ref}}}(s_t)+\beta\log Z(s_t;\beta)
\end{align}    
where the first step follows from the definition of the advantage function $A$ and the definition of the KL-divergence $D_{\mathrm{KL}}$, the second step follows from the definition of the state-action function $Q$ and the value function $V$, the third step follows from the definition of the partition function $Z(s_t; \beta)$, the last step follows from the definition of the KL-divergence. 

According to Eq.~\eqref{eq:optimal_policy}, only the first term $-\beta D_{\mathrm{KL}}(\pi_\theta(z|s_t)\|\frac{\pi_{\mathrm{ref}}(z|s_t)\exp(\beta^{-1}Q_{\pi_{\mathrm{ref}}}(s_t,z))}{Z(s_t;\beta)})$ is the only term contains $\pi_\theta$. Therefore, we can derive the optimal $\pi_\theta$, denoted as $\pi_\theta^*$ as follows:
\begin{align*}
    \pi_\theta^*(z|s_t) = \frac{\pi_{\mathrm{ref}}(z|s_t)\exp(\beta^{-1}Q_{\pi_{\mathrm{ref}}}(s_t,z))}{Z(s_t;\beta)}
\end{align*}

\end{proof}

The following theorem addresses the partition function $ Z(s_t; \beta) $ derived from the optimal policy equation. By leveraging the unique properties of the advantage function $ A $ and the value function $ V $, it effectively mitigates the challenges posed by the partition function.
\begin{theorem} [Offset partition function $Z(s_t, \beta)$, formal version of Theorem~\ref{thm:offsetting_partition:informal}] \label{thm:offsetting_partition}
If the following conditions hold:
\begin{itemize}
    \item Let $c$ denote the prompt used to generate the video. 
    \item Let $x^{<t}$ denote video frames generated before timestep $t$. 
    \item Let $x^t$ denote the video frame generated at timestep $t$. 
    \item Let $Q_\pi, V_\pi, A_\pi$ denote the state-action function, value function, and the advantage function respectively, as defined in Definition~\ref{def:three_essential_function}.
    \item Let $\pi_\theta$ denote the policy being optimized.
    \item Let $\pi_{\mathrm{ref}}$ denote the reference policy.
    \item Let $\beta \in \R$ denote the hyperparameter for controlling the weight of the KL-divergence. 
    \item For simplicity, let $s_t := [c, x^{<t}]$ to represent the state. 
    \item Let $Z([c,x^{<t}];\beta)$ denote the partition function, which is defined by
    \begin{align*}
        Z(s_t;\beta) := \E_{z\sim \pi_{\mathrm{ref}}(\cdot|s_t)}\exp(\beta^{-1}Q_{\pi_{\mathrm{ref}}}(s_t,z))
    \end{align*}
    \item Let $u(c, x_1, x_2)$ denote the difference in rewards of two generated videos $x_1$ and $x_2$, which is defined by
    \begin{align*}
        u(c, x_1, x_2):=\beta\log\frac{\pi_{\theta}(x_1| c)}{\pi_{\mathrm{ref}}(x_1| c)}-\beta\log\frac{\pi_{\theta}(x_2| c)}{\pi_{\mathrm{ref}}(x_2| c)}
    \end{align*}
    \item Let $\delta(c, x_1, x_2)$ denote the difference in sequential forward KL divergence, which is defined by
    \begin{align*}
    \delta(c, x_1, x_2) = & ~~\beta D_{\mathrm{SeqKL}}(c,x_2;\pi_{\mathrm{ref}}\| \pi_{\theta})\\
    &~ -\beta D_{\mathrm{SeqKL}}(c,x_1;\pi_{\mathrm{ref}}\| \pi_{\theta})
    \end{align*}
\end{itemize}

Then, we can show that
\begin{align*}
    P_{\mathrm{BT}}^*(x_1 \succ x_2 |c)=\sigma(u^*(c, x_1, x_2) - \delta^*(c, x_1, x_2))
\end{align*}
\end{theorem}

\begin{proof}

According to Theorem~\ref{thm:equivalence_to_bt_model}, we have
\begin{align} \label{eq:eq_to_bt}
    & ~ P_{\mathrm{BT}}(x_1 \succ x_2 |c) \notag \\
    = & ~ \sigma(\sum_{t=1}^{T_1}\gamma^{t-1}A_\pi([c,x_1^{<t}], x_1^{t}) \notag \\
    & ~ - \sum_{t=1}^{T_2}\gamma^{t-1}A_\pi([c,x_2^{<t}], x_2^{t})). 
\end{align}

According to Theorem~\ref{thm:optimal_policy_for_video_dpo}, we have the following equation
\begin{align*}
    \pi_\theta^*(z|[c,x^{<t}])=
     \frac{\pi_{\mathrm{ref}}(z|[c,x^{<t}])\exp(\beta^{-1} Q_{\pi_{\mathrm{ref}}}([c,x^{<t}],z))}{Z([c,x^{<t}];\beta)}. 
\end{align*}

The above equation can be rearranged to the following format,
\begin{align*} 
    & ~ Q_{\pi_{\mathrm{ref}}}([c,x^{<t}],z) \\
    = & ~ \beta \log\frac{\pi_{\theta}^*(z|[c, x^{<t}])}{\pi_{\mathrm{ref}}(z|[c, x^{<t}])} + \beta \log Z([c, x^{<t}];\beta)
\end{align*}

According to the definition of the advantage function $A$, the state-action function $A$, and the value function $V$, we can have 
\begin{align*}
    &~\sum_{t=1}^{T}\gamma^{t-1}A_{\pi_{\mathrm{ref}}}([c,x^{<t}], x^{t}) \\ 
    = & ~
    \sum_{t=1}^{T}\gamma^{t-1}(Q_{\pi_{\mathrm{ref}}}([c,x^{<t}],x^{t})-V_{\pi_{\mathrm{ref}}}([c,x^{<t}]))\\
    = & ~\sum_{t=1}^{T}\gamma^{t-1}(Q_{\pi_{\mathrm{ref}}}([c,x^{<t}],x^{t}) \\
    & ~ -\E_{z\sim \pi_{\mathrm{ref}}}[Q_{\pi_{\mathrm{ref}}}([c,x^{<t}], z)])\\
    = & ~\sum_{t=1}^{T}\gamma^{t-1}(\beta \log\frac{\pi_{\theta}^*(x^{t}|[c, x^{<t}])}{\pi_{\mathrm{ref}}(x^{t}| [c, x^{<t}])} + \beta \log Z([c, x^{<t}];\beta) \\
    & ~ -\E_{z\sim \pi_{\mathrm{ref}}}[\beta \log\frac{\pi_{\theta}^*(z|[c, x^{<t}])}{\pi_{\mathrm{ref}}(z|[c, x^{<t}])} + \beta \log Z([c, x^{<t}];\beta) ])
\end{align*}
where the first step follows from the definition of the advantage function $A$, the second step follows from the definition of the value function $V$, the third step follows from the definition of the state-action function $Q$. 

On the other hand, we can derive the following,
\begin{align} \label{eq:gamma_A}
    & ~ \sum_{t=1}^{T}\gamma^{t-1}A_{\pi_{\mathrm{ref}}}([c,x^{<t}], x^{t}) \notag \\
    = & ~\beta\sum_{t=1}^{T}\gamma^{t-1}( \log\frac{\pi_{\theta}^*(x^{t}|[c, x^{<t}])}{\pi_{\mathrm{ref}}(x^{t}|[c, x^{<t}])} \notag \\
    & ~ -\E_{z\sim \pi_{\mathrm{ref}}}[\log\frac{\pi_{\theta}^*(z|[c, x^{<t}])}{\pi_{\mathrm{ref}}(z|[c, x^{<t}])}]) \notag \\
    = & ~\beta\sum_{t=1}^{T}\gamma^{t-1}( \log\frac{\pi_{\theta}^*(x^{t}|[c, x^{<t}])}{\pi_{\mathrm{ref}}(x^{t}|[c, x^{<t}])} \notag \\
    & ~ +D_{\mathrm{KL}}(\pi_{\mathrm{ref}}(\cdot|[c,x^{<t}]) \|\pi_{\theta}^*(\cdot|[c,x^{<t}]))) \notag \\
    = & ~\beta\sum_{t=1}^{T} \gamma^{t-1}\log\frac{\pi_{\theta}^*(x^{t}|[c, x^{<t}])}{\pi_{\mathrm{ref}}(x^{t}|[c, x^{<t}])} \notag \\
    & ~ +\beta\sum_{t=1}^{T}\gamma^{t-1}D_{\mathrm{KL}}(\pi_{\mathrm{ref}}(\cdot|[c,x^{<t}]) \|\pi_{\theta}^*(\cdot|[c,x^{<t}])) \notag \\
    = & ~\beta\sum_{t=1}^{T} \log\frac{\pi_{\theta}^*(x^{t}|[c, x^{<t}])}{\pi_{\mathrm{ref}}(x^{t}|[c, x^{<t}])} \notag \\
     & ~ +\beta\sum_{t=1}^{T}D_{\mathrm{KL}}(\pi_{\mathrm{ref}}(\cdot|[c,x^{<t}]) \|\pi_{\theta}^*(\cdot|[c,x^{<t}])) \notag \\
    = & ~\beta(\log\frac{\pi_{\theta}^*(x|c)}{\pi_{\mathrm{ref}}(x|c)}+D_{\mathrm{SeqKL}}(c,x;\pi_{\mathrm{ref}}\| \pi_{\theta}^*))
\end{align}
where the first step follows from the definition of the advantage function $A$, the second step follows from $\E_{z\sim \pi_{\mathrm{ref}}}[\beta \log Z([c, x^{<t}];\beta)] = \beta \log Z([c, x^{<t}];\beta)$, the third step follows from separating the summation operation, the fourth step follows from choosing discount factor $\gamma = 1$, the last step follows from the definition of $D_{\mathrm{SeqKL}}$. 

Reconsidering Eq.~\eqref{eq:eq_to_bt}, and combing Eq.~\eqref{eq:gamma_A} and the definition of $u^*(c, x_1, x_2) $ and $ \delta^*(c, x_1, x_2))$, finally we have
\begin{align*}
    P_{\mathrm{BT}}^*(x_1 \succ x_2 |c)=\sigma(u^*(c, x_1, x_2) - \delta^*(c, x_1, x_2)).
\end{align*}
    
\end{proof}

\end{document}